\documentclass{article} 
\usepackage{iclr2026_conference,times}

\usepackage{amsmath,amsfonts,bm}

\def\eqref#1{equation~\ref{#1}}

\def\1{\bm{1}}

\DeclareMathAlphabet{\mathsfit}{\encodingdefault}{\sfdefault}{m}{sl}
\SetMathAlphabet{\mathsfit}{bold}{\encodingdefault}{\sfdefault}{bx}{n}

\usepackage{hyperref}
\usepackage{url}

\usepackage{wrapfig}
\usepackage{comment}

\usepackage{nicefrac}       
\usepackage{microtype}      
\usepackage{xcolor}         
\usepackage{colortbl} 
\usepackage{subcaption}
\usepackage{multirow}
\usepackage{verbatim}
\usepackage{diagbox}
\usepackage{bm}
\usepackage{makecell}
\usepackage{graphicx}
\usepackage{amsmath}
\usepackage[english]{babel}
\usepackage{mathtools}
\usepackage[ruled,vlined]{algorithm2e}
\usepackage[utf8]{inputenc}
\usepackage{csquotes}
\usepackage{enumitem}
\usepackage{soul}
\usepackage{lipsum} 

\usepackage{booktabs} 
\usepackage{caption} 
\usepackage{rotating} 
\usepackage{multirow} 
\usepackage{lipsum} 
\usepackage{colortbl}
\usepackage{arydshln}
\usepackage{amsthm}
\usepackage{amssymb}
\usepackage{bbm}
\newtheorem{proposition}{Proposition}

\definecolor{smoothred}{RGB}{220, 50, 47} 
\definecolor{smoothblue}{RGB}{68, 90, 225} 
\definecolor{smoothgreen}{RGB}{80, 160, 90}
\definecolor{text-green}{RGB}{91,124,131}
\definecolor{text-yellow}{RGB}{168,146,89}

\definecolor{table-yellow}{RGB}{250, 223, 161}
\definecolor{table-blue}{RGB}{173, 216, 230}
\definecolor{table-green}{RGB}{126, 172, 181}
\definecolor{darkblue}{rgb}{0, 0, 0.5}

\title{PACR: Progressively Ascending Confidence Reward for LLM Reasoning}

\author{Eunseop Yoon$^{1}$\,\,\,\,\,\,\,\,\,\, Hee Suk Yoon$^{1}$\thanks{Work done during an internship at Microsoft Research Asia.}\,\,\,\,\,\,\,\,\,\, Jaehyun Jang$^{1}$\,\,\,\,\,\,\,\,\,\, SooHwan Eom$^{1}$\,\,\,\,\,\,\,\,\,\, \\\bf{Qi Dai}$^{2}$\,\,\,\,\,\,\,\,\,\, \bf{Chong Luo}$^{2}$\,\,\,\,\,\,\,\,\,\, \bf{Mark Hasegawa-Johnson}$^{3}$\,\,\,\,\,\,\,\,\,\, \bf{Chang D. Yoo}$^{1}$\thanks{Corresponding Author} \\
$^{1}$Korea Advanced Institute of Science and Technology (KAIST)\,\,\,\,\,\,\,\,\,\,\\$^{2}$Microsoft Research Asia (MSRA)\,\,\,\,\,\,\,\,\,\,$^{3}$University of Illinois at Urbana-Champaign (UIUC)\\
}

\iclrfinalcopy 
\begin{document}

\maketitle

\begin{abstract}
Reinforcement Learning with Verifiable Rewards (RLVR) has significantly improved LLM reasoning, but its sparse, outcome-based reward provides no guidance for intermediate steps, slowing exploration. We propose \underline{\textbf{P}}rogressively \underline{\textbf{A}}scending \underline{\textbf{C}}onfidence \underline{\textbf{R}}eward (\textbf{PACR}), a dense, model-intrinsic reward computed directly from the model’s evolving belief in the correct answer. PACR encodes the inductive bias that, along a well-formed reasoning trajectory, the probability of the ground-truth answer should have a generally ascending trend. We provide empirical and theoretical analysis validating that such an inductive bias constrains the exploration search space to regions richer in logically sound reasoning. We demonstrate that PACR accelerates exploration, reaches reward saturation with fewer trajectories, and yields improvements on multiple benchmarks. Our results suggest that dense, model-intrinsic shaping signals can make RLVR training more effective and reliable. Code will be released.

\end{abstract}

\section{Introduction}
\begin{wrapfigure}{r}{0.5\textwidth}
    \vskip -0.5in
    \centering    
    \includegraphics[width=0.4\textwidth]{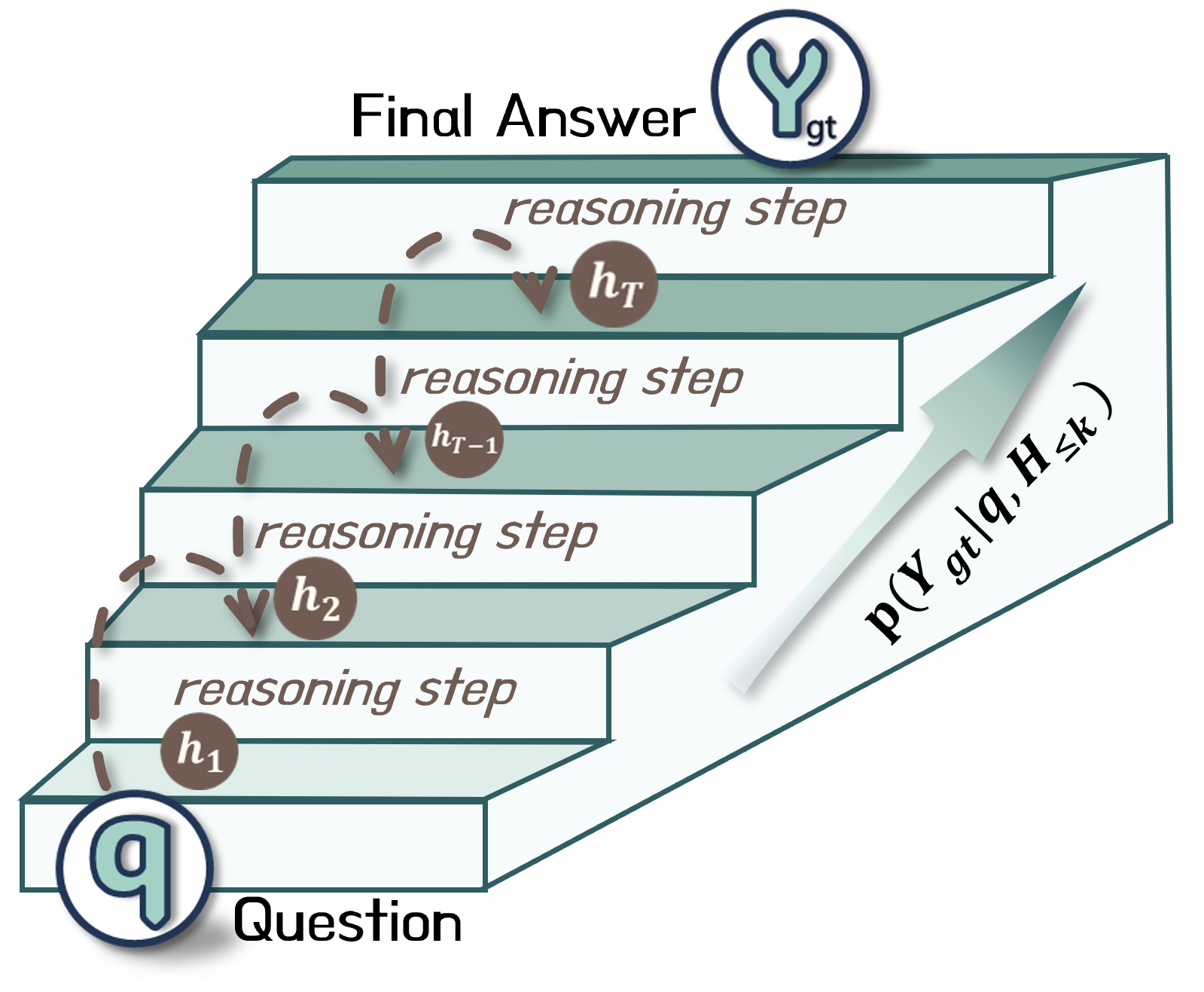}
    \vskip -0.1in
\caption{\textbf{Stepwise confidence growth.} For a question $q$, a well-formed sequence of reasoning steps $h_1,\ldots,h_k$ should increase the model's probability of the ground-truth answer $Y_{\mathrm{gt}}$ across steps.} 

    \label{fig:intro}
    \vskip -0.1in
\end{wrapfigure}
Pre-trained large language models (LLMs) exhibit strong performance on complex, multi-step reasoning tasks \citep{gemini2.5, model_qwen3, qwq32b}. Reinforcement Learning with Verifiable Rewards (RLVR) has emerged as a leading approach for further improving such capabilities, using a programmatically checkable terminal metric (e.g., exact-match on the final answer) as the reward \citep{grpo, deepseek-r1}. While effective, the standard RLVR formulation supplies a sparse terminal accuracy signal, offering no guidance for intermediate steps and thus exacerbating credit assignment. Alternative process-based supervision employs external reward models to score intermediate reasoning, but is costly to train, data-hungry, and prone to misalignment \citep{prime, prm2}.

This work asks whether we can obtain \emph{stepwise supervision} directly from the model. Psycholinguistic work shows that people interpret language incrementally, updating expectations with each word; as context accumulates, uncertainty falls and the correct interpretation becomes more likely \citep{psycholinguistic1, psycholinguistic2}. By the same logic, in tasks with a verifiable final answer, a correct intermediate step should typically raise the model’s probability of the ground-truth answer. Concretely, given a question \(q\), a reasoning prefix \(H_{\le k}\), and ground truth \(Y_{\mathrm{gt}}\), we track the model’s confidence \(p\!\left(Y_{\mathrm{gt}}\mid q, H_{\le k}\right)\) and expect a general upward trend over steps (Figure \ref{fig:intro}).

Guided by this premise, we introduce the \textbf{Progressively Ascending Confidence Reward (PACR)}, a dense, model-intrinsic signal that converts confidence growth into stepwise supervision for LLM reasoning during reinforcement learning. During training, as the model produces a sequence of reasoning steps for a question with a verifiable answer, we evaluate at each step the log-probability assigned to the ground-truth answer and reward any positive change, effectively encouraging a consistently upward trend in confidence. Because PACR is computed from the model’s own probabilities, it requires no external reward model and is available at every step, improving credit assignment and steering search toward faithful trajectories. We pair PACR with the standard RLVR terminal accuracy reward so the objective remains anchored to verifiable correctness while the process signal shapes the reasoning path. \textbf{In detail, our contributions can be summarized as follows:}

\begin{itemize}[left=0em]
    \setlength\itemsep{0em}
\item \textbf{Empirical Validation of an Inductive Bias (Section \ref{sec:observation}).} We provide extensive observational evidence that ground-truth confidence growth acts as a powerful inductive bias. Our analyses on open-source LLMs reveal three key findings: (1) a \emph{consistent} confidence ascent strongly correlates with final answer correctness; (2) among correct answers, logically coherent reasoning paths exhibit an even \emph{more consistent} ascent than spurious ones; and (3) the \emph{magnitude} of the confidence gain effectively pinpoints pivotal reasoning steps.

\item \textbf{Theoretical Justification (Section \ref{sec:theoretical} and \ref{method}).} We provide a theoretical foundation for using confidence growth as a process reward. We prove that a reasoning step from an idealized oracle policy will, on average, increase or maintain the model's confidence in the ground truth, validating it as a strong inductive bias. Building on this, we formalize the \textbf{Progressively Ascending Confidence Reward (PACR)} and introduce two concrete methods for its implementation: \textbf{Sparse-PACR} for trajectory-level rewards and \textbf{Dense-PACR} for step-wise rewards.

\item \textbf{Experimental Results (Section \ref{sec:result}).} Across multiple reasoning benchmarks, augmenting RLVR with our PACR methods improves training dynamics and final performance. Our approach accelerates exploration and ultimately attains a higher, more consistent final score than the baseline, demonstrating a more effective and reliable training process.
\end{itemize}

\section{Related Work}
\label{related_work}
\paragraph{Outcome-based RL for LLM Reasoning}

Reinforcement Learning (RL) is increasingly used to fine-tune Large Language Models (LLMs). This is done not only to align models with human preferences via Reinforcement Learning from Human Feedback (RLHF) \citep{rlhf, remax, hhrlhf} but also to enhance their reasoning abilities for complex problem-solving \citep{self_correction}. To improve these reasoning capabilities, a recent prominent approach is Reinforcement Learning with Verifiable Reward (RLVR) \citep{deepseek-r1, model_qwen_math, deepseek_math}, which uses an outcome-based reward instead of a proxy reward model. For example, a reward of 1 is assigned for a correct answer and 0 (or -1) for an incorrect one. Then, the model generates multiple trajectories for a single problem. The reward for each trajectory is then compared against the average reward across all samples in the group, and this relative reward is used as an advantage to train the model. This outcome-based reward system is widely explored \citep{drgrpo, dapo, openreasoner, simplerl} because it is easily scalable, and mitigates concerns about reward hacking by eliminating the need for a separate reward model \citep{deepseek-r1}. 

\paragraph{Dense Reward for LLMs Finetuning with RL}
To overcome the limitations of holistic, trajectory-level sparse rewards, various approaches for providing dense rewards have been explored. In RLHF, for instance, approaches include training an external reward model to assign token-level rewards using synthesized data \citep{tlcr}, utilizing a more mature external model as the reward model \citep{drlc, finegrained}, and use implicit reward signal from reward model \citep{abc}. Similarly, direct alignment algorithms (e.g., DPO \citep{DPO}) have been adapted to provide dense rewards by re-framing DPO's implicit reward at a token level \citep{tdpo, tgdpo, rto, densedpo} or by selectively using specific tokens for the reward signal \citep{confpo, tis-dpo}. For training a reasoning LLM with RL, previous approaches include training a Process Reward Model (PRM) for process-level rewards \citep{prm_qvalue, prm2, progrm}, or defining a DPO-like implicit reward at the token level \citep{prime, freeprocess}. 

\section{Background and Problem Setup}
\label{preliminary}
This section introduces the notation for reasoning trajectories, how we segment and evaluate stepwise confidence in the ground-truth answer, and the RL objective we use in training.
\paragraph{Reasoning Trajectories and Notation.} Given a question $q$, a policy $\pi_\theta$ generates a \emph{sequence of reasoning steps} $H=(h_1,\ldots,h_T)$ and a final answer $\hat{Y}$. Let $Y_{\mathrm{gt}}$ denote the verifiable ground-truth answer. We write $H_{\le k}=(h_1,\ldots,h_k)$ for the reasoning steps up to step $k$. We analyze and shape the reasoning process by tracking how the model’s probability of $Y_{\mathrm{gt}}$ evolves with the prefix $H_{\le k}$.

\paragraph{Segmenting the Reasoning Process and Stepwise Ground-truth Confidence.}
Similar to \citet{segmenting1}, we segment each generated reasoning trace into discrete steps $\{h_k\}_{k=1}^{T}$ using a simple, model-agnostic rule: start a new step at a newline (`\texttt{\textbackslash n}') or at a period followed by a space (`\texttt{. }'); fragments shorter than five tokens are merged with the preceding step to avoid overly fine splits. To measure ground-truth–anchored confidence at step $k$, we standardize the answer format by appending a short prefix $y^{0}_{\mathrm{gt}}$ (e.g., `\texttt{So the final answer is \textbackslash boxed\{}') and evaluate the model’s probability of the ground-truth answer $Y_{\mathrm{gt}}=(y^1_{\mathrm{gt}},\ldots,y^L_{\mathrm{gt}})$ under the current prefix $H_{\le k}$. Writing $Y_{\mathrm{gt}}=(y^1_{\mathrm{gt}},\ldots,y^L_{\mathrm{gt}})$, we measure the
\textbf{ground-truth confidence} at step $k$ as
\begin{equation}
\log p(Y_{gt} | q, H_{\leq k}) \;=\; \sum_{l=1}^{L} \log p_\theta\!\big(y^l_{\mathrm{gt}} \,\big|\, q,\, H_{\le k},\, y^{0}_{\mathrm{gt}},\, y^{<l}_{\mathrm{gt}}\big),
\end{equation}
where $y^{<l}_{\mathrm{gt}}$ are preceding answer tokens. This measures the model's confidence in the ground truth answer at any given stage of its reasoning steps. 

\paragraph{Group Relative Policy Optimization (GRPO)}
GRPO \citep{grpo} estimates advantages by comparing returns \emph{within} a group of $N$ samples rather than using a learned value function. For a given question $q$ (with verifiable answer $Y_{\mathrm{gt}}$), the behavior policy $\pi_{\theta_{\mathrm{old}}}$ generates $N$ trajectories
\begin{equation}
\{\tau^{(i)}\}_{i=1}^{N}, \qquad 
\tau^{(i)} \;=\; \big(H^{(i)},\, \hat{Y}^{(i)}\big),
\end{equation}
where $H^{(i)}=(h^{(i)}_1,\ldots,h^{(i)}_{T_i})$ are the reasoning steps, $T_i$ is the number of steps for $i$-th trajectory  and $\hat{Y}^{(i)}$ is the predicted answer for $i$-th trajectory. 

For each sampled trajectory $i$, we compare the predicted answer $\hat{Y}^{(i)}$ with the ground truth $Y_{\mathrm{gt}}$ and assign a binary terminal accuracy reward:
\begin{equation}
R^{(i)} \;=\;
\begin{cases}
1, & \texttt{is\_equivalent}\!\big(\hat{Y}^{(i)},\, Y_{\mathrm{gt}}\big) \\
0, & \text{otherwise.}
\end{cases}
\end{equation}
Here \texttt{is\_equivalent} performs task-specific normalization (e.g., stripping whitespace/punctuation, handling LaTeX boxing, case folding, and numeric tolerances) before exact match. The group-relative advantage for trajectory $i$ is computed by centering (and optionally standardizing) its reward within the cohort of $N$ samples:
\begin{equation}
A^{(i)} = \frac{R^{(i)} - \text{mean}(\{R^{(i)}\}_{i=1}^N)}{\textcolor{smoothgreen}{\text{std}(\{R^{(i)}\}_{i=1}^N)}}.
\label{eq:grpo_adv}
\end{equation}

Similar to PPO \citep{ppo}, GRPO adopts a clipping  with KL penalty:
\scriptsize
\begin{equation}
\begin{aligned}
&\mathcal{J}_\text{GRPO}(\theta)=\\ &\mathop{\mathbb{E}}\limits_{\substack{(q,Y_{gt})\sim \mathcal{D}\\ \{\tau^{(i)}\}\sim \pi_{\theta_{\text{old}}}(\cdot\mid q)}}
\Bigg[ \frac{1}{N}\sum_{i=1}^{N} \textcolor{smoothgreen}{\frac{1}{|\tau^{(i)}|}} \Bigg( 
\min \Big( \frac{\pi_{\theta}(\tau^{(i)} \mid q)}{\pi_{\theta_{\text{old}}}(\tau^{(i)} \mid q)}(\theta) {A}^{(i)},  
\ \text{clip} \Big( \frac{\pi_{\theta}(\tau^{(i)} \mid q)}{\pi_{\theta_{\text{old}}}(\tau^{(i)} \mid q)}, 1 - \varepsilon, 1 + \varepsilon \Big) {A}_{i} \Big)
- \beta D_{\text{KL}}(\pi_{\theta} || \pi_{\text{ref}}) 
\Bigg) \Bigg],
\label{eq:grpoloss}
\end{aligned}
\end{equation}
\normalsize
where $\mathcal{D}$ is the training dataset and $\pi_{\text{ref}}$ is a reference policy.
In our work, we follow the Dr. GRPO \citep{drgrpo} formulation, a bias-mitigated variant of GRPO. This approach modifies the standard GRPO algorithm by discarding the standard deviation from the advantage calculation and the length normalization from the loss function (the terms shown in \textcolor{smoothgreen}{\textbf{green}} in Eq. \ref{eq:grpo_adv} and Eq. \ref{eq:grpoloss}).

\section{Is Ground-Truth Confidence Growth a Useful Inductive Bias?}
\label{observation}
In this section, we present both empirical observations and a theoretical justification to assess whether ground-truth confidence growth constitutes a useful inductive bias for training LLM reasoning.
\paragraph{Ground-truth Confidence Growth.} We first quantify confidence growth by defining the stepwise confidence gain, $C_{k}$, as the change in the log-probability of the ground-truth answer induced by the addition of reasoning step $h_k$:
\begin{equation}
C_{k} \;\coloneqq\; \log \pi_\theta\!\big(Y_{\mathrm{gt}} \mid q, H_{\le k}\big)\;-\;\log \pi_\theta\!\big(Y_{\mathrm{gt}} \mid q, H_{<k}\big),
\label{eq:C_k}
\end{equation}
where $H_{\le k}=(h_1,\ldots,h_k)$ and $H_{<k}=(h_1,\ldots,h_{k-1})$. For $k=1$, $H_{<1}$ is the empty prefix. 
Intuitively, $C_{k} > 0$ indicates that step $h_k$ makes the ground truth more probable, whereas $C_{k} < 0$ indicates the opposite. (When indexing trajectories, we write $C^{(i)}_k$.) For brevity, we will hereafter use ``\textit{confidence growth}" and ``\textit{ground-truth confidence growth}" interchangeably.

\subsection{Observing Ground-truth Confidence Growth on Reasoning Models}
\label{sec:observation}

\paragraph{Observation 1: Consistent Confidence Ascent Correlates with Final Correctness.}

\begin{wrapfigure}{r}{0.4\textwidth}
    \vskip -0.2in
    \centering    
    \includegraphics[width=0.4\textwidth]{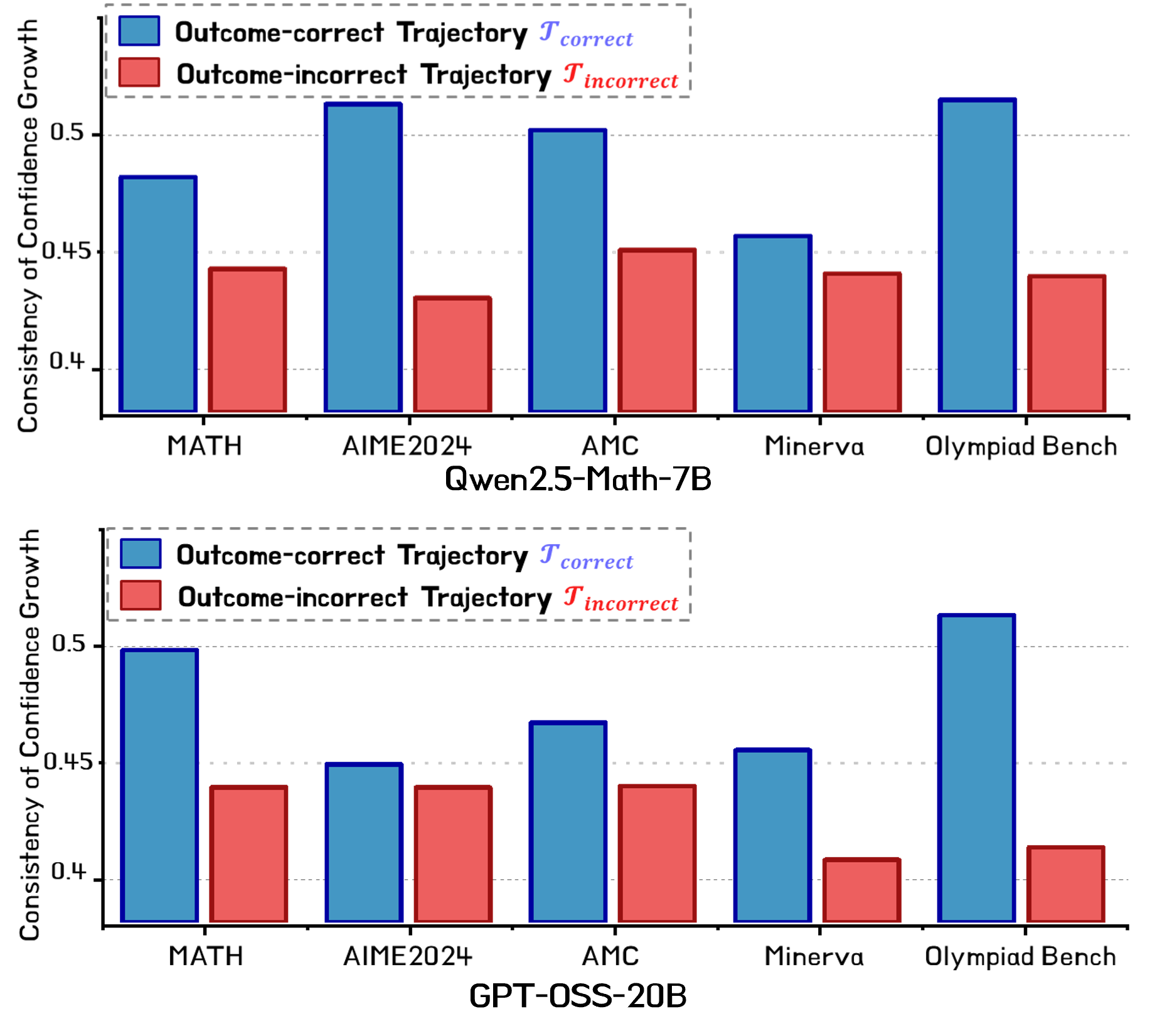}
    \vskip -0.1in
\caption{\textbf{Consistency of confidence growth correlates with correctness.} Outcome-correct trajectories (\textcolor{smoothblue}{$\mathcal{T}_{\text{correct}}$}) exhibit a higher proportion of steps with positive confidence gain ($C_k > 0$) compared to incorrect trajectories (\textcolor{smoothred}{$\mathcal{T}_{\text{incorrect}}$}).}
    \label{fig: obs1}
    \vskip -0.1in
\end{wrapfigure} 
To empirically validate the connection between ground-truth confidence growth and final outcome accuracy, we analyzed trajectories generated by several strong open-source LLMs, such as Qwen2.5-Math.7B \citep{model_qwen_math} and GPT-OSS-20B \citep{gpt-oss-20b}. Specifically, we prompted the models to generate a single reasoning trajectory ($\tau$) for each question in a large set of reasoning tasks. This entire collection of trajectories was then partitioned into two distinct sets: the set of \textcolor{smoothblue}{outcome-correct trajectores, $\mathcal{T}_{\text{correct}}$}, where the final answer matched the ground truth, and the set of \textcolor{smoothred}{outcome-incorrect trajectories, $\mathcal{T}_{\text{incorrect}}$}.We then calculated the \textbf{consistency of confidence growth}, defined as the proportion of positive-gain steps for each trajectory (i.e., $\frac{1}{T}\sum_{k=1}^T \mathbb{I}(C_k>0)$, where $\mathbb{I}$ is the indicator function). Our findings reveal a clear distinction between the two groups. As illustrated in Figure \ref{fig: obs1}, trajectories in \textcolor{smoothblue}{$\mathcal{T}_{\text{correct}}$} exhibited a higher proportion of steps with positive confidence gain compared to trajectories in \textcolor{smoothred}{$\mathcal{T}_{\text{incorrect}}$}. \textit{This indicates that reasoning paths that result in a correct answer tend to be characterized by a more consistent, progressive increase in the model's belief in the ground-truth answer.}
\paragraph{Observation 2: Coherent Reasoning Paths Exhibit More Consistent Confidence Ascent.}
\begin{wrapfigure}{r}{0.4\textwidth}
    
    \vskip -0.2in
    \centering    
    \includegraphics[width=0.35\textwidth]{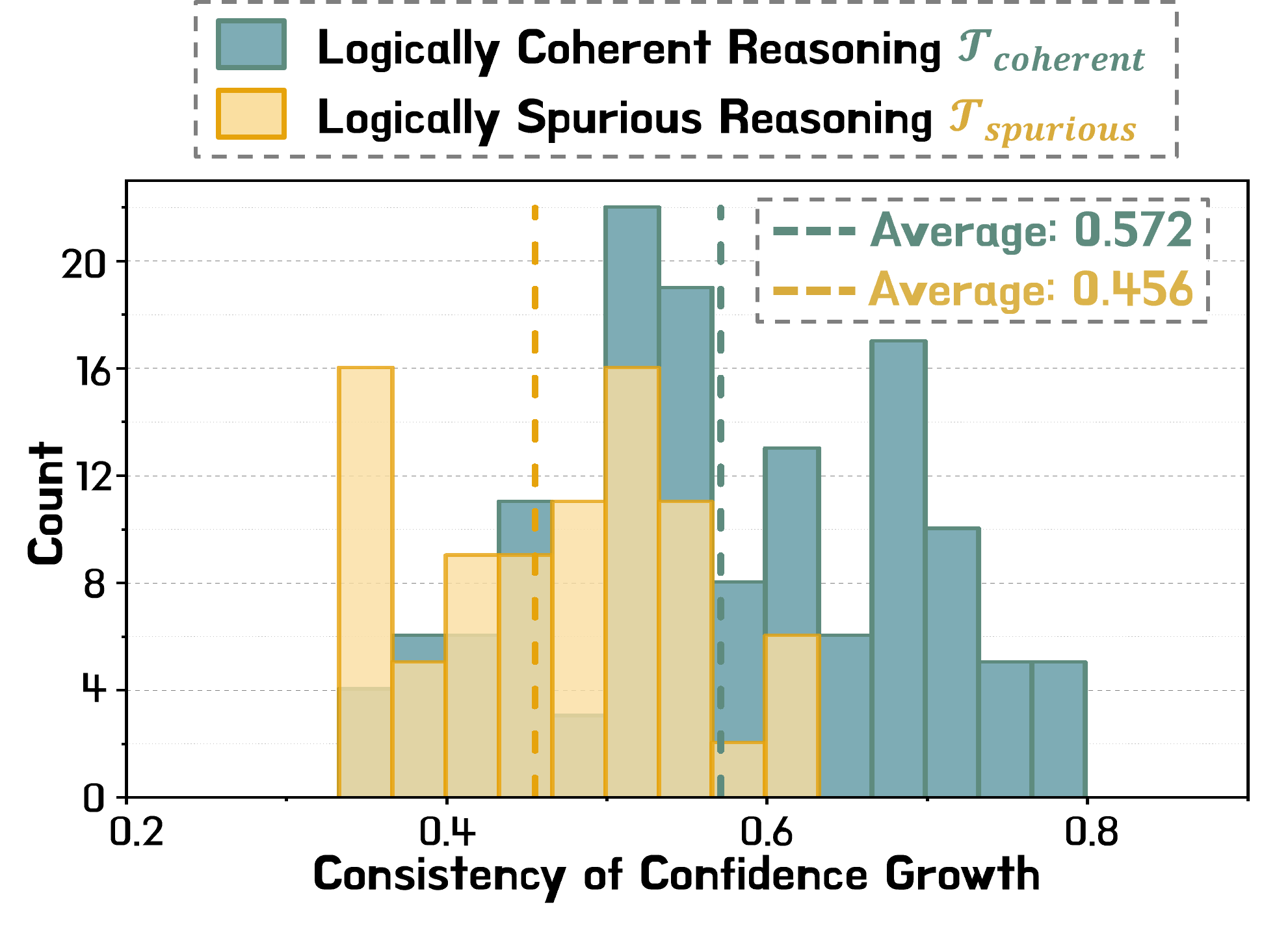}
    \vskip -0.1in
\caption{\textbf{Consistency of confidence growth reflects reasoning quality.} Coherent reasoning paths (\textcolor{text-green}{$\mathcal{T}_{\text{coherent}}$}) show a more consistent confidence ascent (higher proportion of $C_k > 0$ steps) than spurious paths (\textcolor{text-yellow}{$\mathcal{T}_{\text{spurious}}$}), even when both yield the correct final answer.} 

    \label{fig: obs2}
    \vskip -0.1in
\end{wrapfigure} 
While Observation 1 established a link between confidence growth and correct outcomes, we next sought to determine if this signal was also sensitive to the \emph{quality of the reasoning process}. A trajectory can arrive at the correct answer through flawed or spurious steps, and a robust process-level signal should be able to distinguish such cases. To investigate this, we focused exclusively on the set \textcolor{smoothblue}{$\mathcal{T}_{\text{correct}}$}. We employed a powerful external LLM evaluator (GPT-5) to further partition this set into two subgroups: those with \textcolor{text-green}{logically coherent reasoning $\mathcal{T}_{\text{coherent}}$} and those with \textcolor{text-yellow}{spurious reasoning $\mathcal{T}_{\text{spurious}}$}, where the correct answer was reached via flawed logic or irrelevant steps (see Appendix \ref{appendix:prompt} for detailed evaluation prompts). As shown in Figure \ref{fig: obs2}, the average proportion of positive-gain steps was significantly higher for \textcolor{text-green}{$\mathcal{T}_{\text{coherent}}$} compared to \textcolor{text-yellow}{$\mathcal{T}_{\text{spurious}}$}. This demonstrates that while both groups reached the correct final answer, the model’s confidence grew more consistently when following a logically sound path. \textit{This finding suggests that confidence ascent is not merely an indicator of the final outcome but also a signal reflecting the internal quality of the reasoning trace itself.}

\begin{figure}
\centering
    \includegraphics[width=1.0\textwidth]{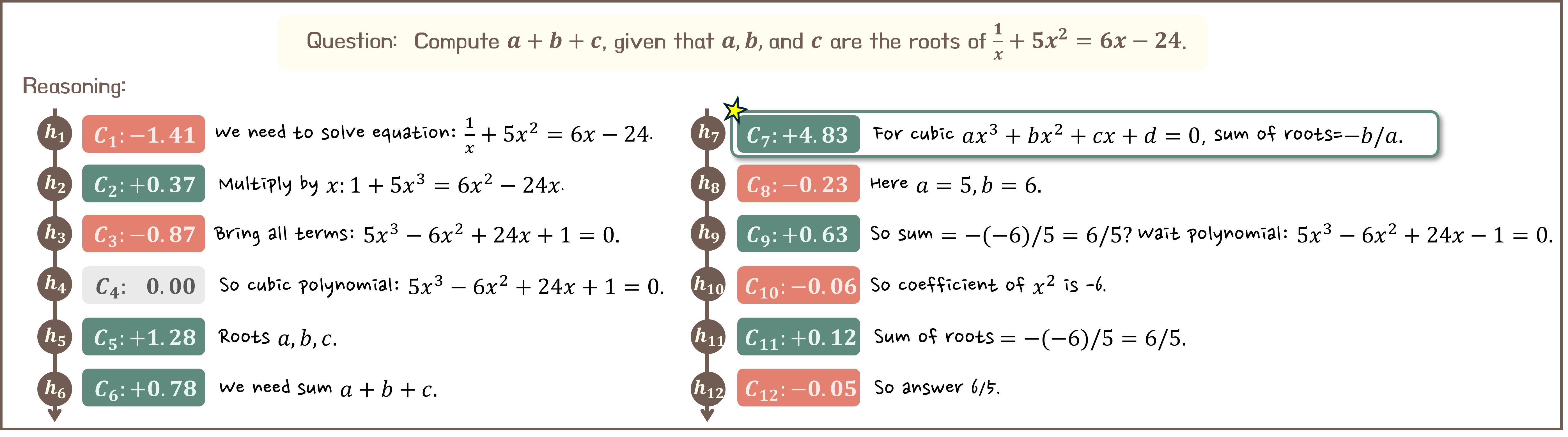}
\vskip -0.1in
\caption{\textbf{Qualitative example of a pivotal step.} Among the reasoning steps, a critical insight at step $h_7$ (the introduction of Vieta's formulas for a cubic equation) results in a large, distinct spike in the ground-truth confidence gain ($C_7 = +4.83$). This is substantially larger than the gains from more routine algebraic steps. Further qualitative examples are provided in Appendix \ref{appendix:obs3_qual}.} 
\vskip -0.1in
    \label{fig:observation3_qual}
\end{figure}

\paragraph{Observation 3: Large Stepwise Confidence Gains Pinpoint Pivotal Reasoning Steps.}
\begin{wrapfigure}{r}{0.4\textwidth}
    \vskip -0.2in
    \centering    
    \includegraphics[width=0.4\textwidth]{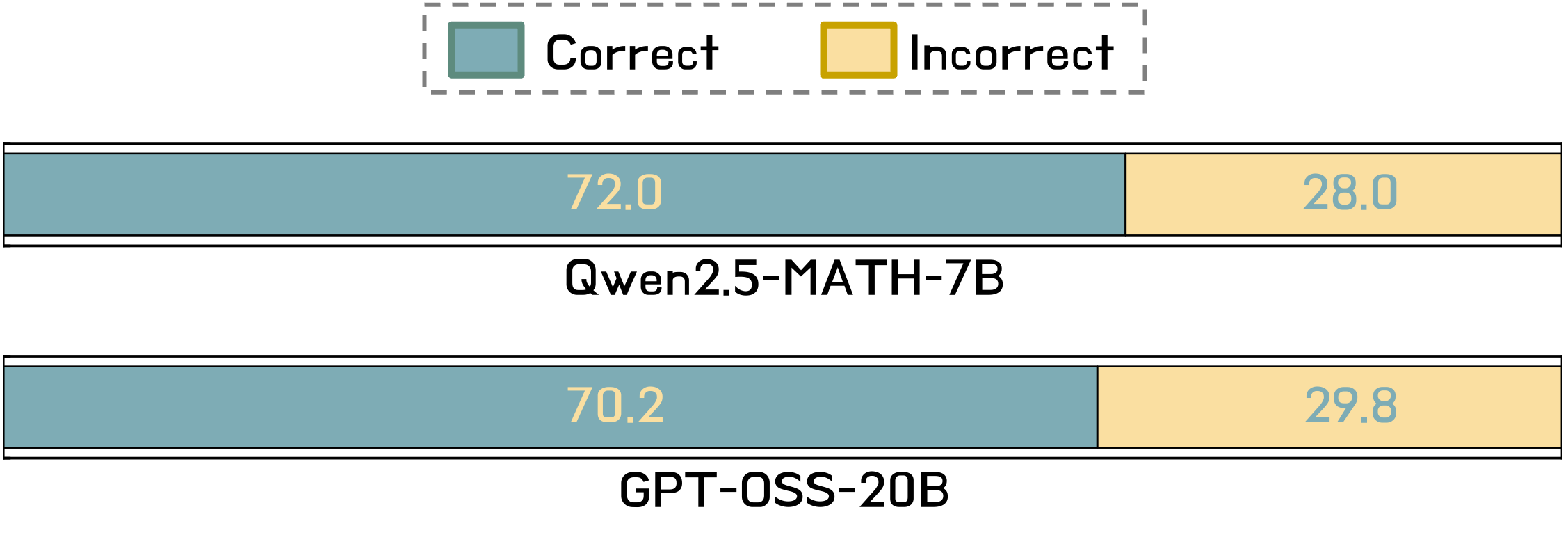}
    \vskip -0.1in
\caption{\textbf{Quantitative validation of step importance.} In a pairwise comparison, an LLM evaluator judged the step with the higher confidence gain ($C_i > C_j$) as more critical with a win rate significantly above chance, confirming that gain magnitude correlates with step importance.}
    \label{fig: obs3_quant}
    \vskip -0.1in
\end{wrapfigure} 
Beyond the overall trend of confidence, we investigated whether the \emph{magnitude} of the stepwise gain, $C_k$, correlates with the importance of individual reasoning steps. Qualitatively, we observe that large, positive spikes in $C_k$ often coincide with pivotal moments in the reasoning process, such as the application of a key theorem or a critical insight. For instance, as illustrated in Figure \ref{fig:observation3_qual}, a step introducing the sum of roots formula for a cubic equation yields a substantially larger confidence gain compared to adjacent steps involving routine algebraic manipulation. To validate this rigorously, we conducted a quantitative pairwise comparison. For trajectories in \textcolor{smoothblue}{$\mathcal{T}_{\text{correct}}$}, we randomly sampled pairs of reasoning steps, $h_i$ and $h_j$, under the condition that $C_i > C_j$. We then prompted an LLM evaluator (GPT-5) to judge which of the two steps was more critical for reaching the final solution (see Appendix \ref{appendix:prompt} for detailed evaluation prompts). The step with the higher confidence gain, $h_i$, was frequently identified as more critical, achieving a win rate significantly above chance (Figure \ref{fig: obs3_quant}). This finding suggests that the magnitude of the confidence gain is not arbitrary; it is a meaningful signal that effectively pinpoints influential steps within a reasoning chain. \textit{This provides a strong rationale for using it as a training objective, as maximizing $C_k$ would directly incentivize the model to generate these critical, problem-solving actions.}

\subsection{Theoretical Justification for Ground-Truth Confidence Growth as a Process Reward}
\label{sec:theoretical}
Building on our empirical findings, we now provide a theoretical foundation for using confidence growth as a process reward. We prove that a reasoning step sampled from an ideal ``oracle" policy (one that generates steps consistent with the ground truth, i.e., faithful steps) will, on average, increase or maintain the model's confidence in that ground truth.

\paragraph{The Oracle Policy Assumption} Our theoretical analysis is built on the following assumption: a capable LLM, when conditioned on a correct final answer, is able to generate a faithful and logically sound reasoning path. This assumption is well-founded, as modern LLMs excel at rationalization; their training enables them to construct coherent explanations that bridge a given question and its answer. Based on this premise, we can construct an idealized model for analysis, which we term an oracle policy $\pi_{\text{oracle}}$. This policy is the model's own generative process given access to the ground-truth answer $Y_{\text{gt}}$, sampling the next step $h_k$ from the distribution $\pi_\theta(h_k|q, Y_{\text{gt}}, H_{<k})$.

\begin{proposition}
Let $C_k$ be the stepwise confidence gain at step $k$. The expected value of $C_k$ is non-negative when the expectation is taken over reasoning steps $h_k$ sampled from the oracle policy. Formally:
$$
\mathbb{E}_{h_k \sim \pi_\theta(\cdot | q, Y_{\mathrm{gt}}, H_{<k})} [C_k] \ge 0.
$$
\end{proposition}

\begin{proof}
We begin with the definition of the ground-truth confidence growth $C_k$, as defined in Eq. \ref{eq:C_k}:
$$
\begin{aligned}
C_k &= \log \pi_\theta(Y_{\mathrm{gt}} | q, H_{\le k}) - \log \pi_\theta(Y_{\mathrm{gt}} | q, H_{<k}) \\
&= \log \frac{\pi_\theta(Y_{\mathrm{gt}} | q, h_k, H_{<k})}{\pi_\theta(Y_{\mathrm{gt}} | q, H_{<k})}.
\end{aligned}
$$
Next, we apply Bayes' rule to the numerator, $\pi_\theta(Y_{\mathrm{gt}} | q, h_k, H_{<k})$:
$$
\pi_\theta(Y_{\mathrm{gt}} | q, h_k, H_{<k}) = \frac{\pi_\theta(h_k | q, Y_{\mathrm{gt}}, H_{<k}) \pi_\theta(Y_{\mathrm{gt}} | q, H_{<k})}{\pi_\theta(h_k | q, H_{<k})}.
$$
Substituting this back into the equation for $C_k$, the $\pi_\theta(Y_{\mathrm{gt}} | q, H_{<k})$ terms cancel, yielding:
$$
C_k = \log \frac{\pi_\theta(h_k | q, Y_{\mathrm{gt}}, H_{<k})}{\pi_\theta(h_k | q, H_{<k})}.
$$
Now, we take the expectation of $C_k$ with respect to the oracle policy, $\pi_\theta(\cdot | q, Y_{\mathrm{gt}}, H_{<k})$:
$$
\begin{aligned}
\mathbb{E}_{h_k \sim \pi_\theta(\cdot | q, Y_{\mathrm{gt}}, H_{<k})} [C_k] &= \sum_{h_k} \pi_\theta(h_k | q, Y_{\mathrm{gt}}, H_{<k}) \log \frac{\pi_\theta(h_k | q, Y_{\mathrm{gt}}, H_{<k})}{\pi_\theta(h_k | q, H_{<k})} \\
&= D_{KL}\Big( \pi_\theta(\cdot | q, Y_{\mathrm{gt}}, H_{<k}) \ || \ \pi_\theta(\cdot | q, H_{<k}) \Big).
\end{aligned}
$$
This expression is the Kullback-Leibler (KL) divergence between the probability distribution over the next reasoning step conditioned on the ground truth and the distribution without it. By the property of non-negativity of KL divergence, the proposition holds:
$$
D_{KL}(\cdot || \cdot) \ge 0 \quad \therefore \quad \mathbb{E}[C_k] \ge 0 \quad
$$
\end{proof}

\paragraph{Implication.}
This proof demonstrates that the expected confidence gain under an oracle policy is equivalent to the KL divergence between the ground-truth-conditioned policy and the standard policy. Since KL divergence is always non-negative, this provides the following theoretical guarantee: on average, a reasoning step consistent with the correct answer (i.e., faithful reasoning) will not decrease the model's confidence. \textit{This result validates the use of confidence growth as a strong inductive bias; encouraging the model to explore reasoning paths with non-decreasing confidence effectively constrains the search space to regions richer in logically sound reasoning.}

\section{Method: Progressively Ascending Confidence Reward (PACR)}
\label{method}
Based on our findings in Section \ref{observation}, we now formalize how to incorporate the principle of ascending ground-truth confidence into the GRPO framework. To do this, we introduce the Progressively Ascending Confidence Reward (PACR), a procedural reward signal designed to complement the final outcome-based reward. We propose \textit{two} variants: (1) \textcolor[RGB]{95, 164, 207}{\textbf{Sparse-PACR}}, which applies a holistic, trajectory-level reward based on the consistency of confidence growth, and (2) \textcolor[RGB]{231, 143, 129}{\textbf{Dense-PACR}}, which provides a fine-grained, step-wise reward based on the magnitude of each confidence change.

\paragraph{\textcolor[RGB]{95, 164, 207}{Sparse-PACR.}}
In the Sparse setting, we compute a single procedural reward for an entire trajectory based on the consistency of its confidence growth. This reward, $C^{(i)}_{\text{sparse}}$, is the proportion of reasoning steps that produce a positive confidence gain. We calculate it using an indicator function, $\mathbb{I}(\cdot)$:
\begin{equation}
C^{(i)}_{\text{sparse}} = \frac{1}{T_i} \sum_{k=1}^{T_i} \mathbb{I}\Big(C_k^{(i)} > 0\Big),
\label{eq:sparse_reward}
\end{equation}
where $C_k^{(i)}$ is the confidence gain in Eq.~\ref{eq:C_k}. The final reward for trajectory $i$, $R^{(i)}_{\text{sp-PACR}}$, is a weighted combination of the standard outcome-based reward, $R^{(i)}_{\text{GRPO}}$, and our sparse procedural reward:
\begin{equation}
R^{(i)}_{\text{sp-PACR}} = \lambda_1 \cdot R^{(i)}_{\text{GRPO}} + \lambda_2 \cdot C^{(i)}_{\text{sparse}}.
\end{equation}
This combined reward is then used to calculate the trajectory's advantage, $A^{(i)}_{\text{sp-PACR}}$, within the GRPO framework by centering it against the group average:
\begin{equation}
A^{(i)}_{\text{sp-PACR}} = R^{(i)}_{\text{sp-PACR}} - \text{mean}\big(\{R^{(j)}_{\text{sp-PACR}}\}_{j=1}^N\big).
\label{eq:sp-PACR_adv}
\end{equation}

\begin{figure}
\centering
    \includegraphics[width=1.0\textwidth]{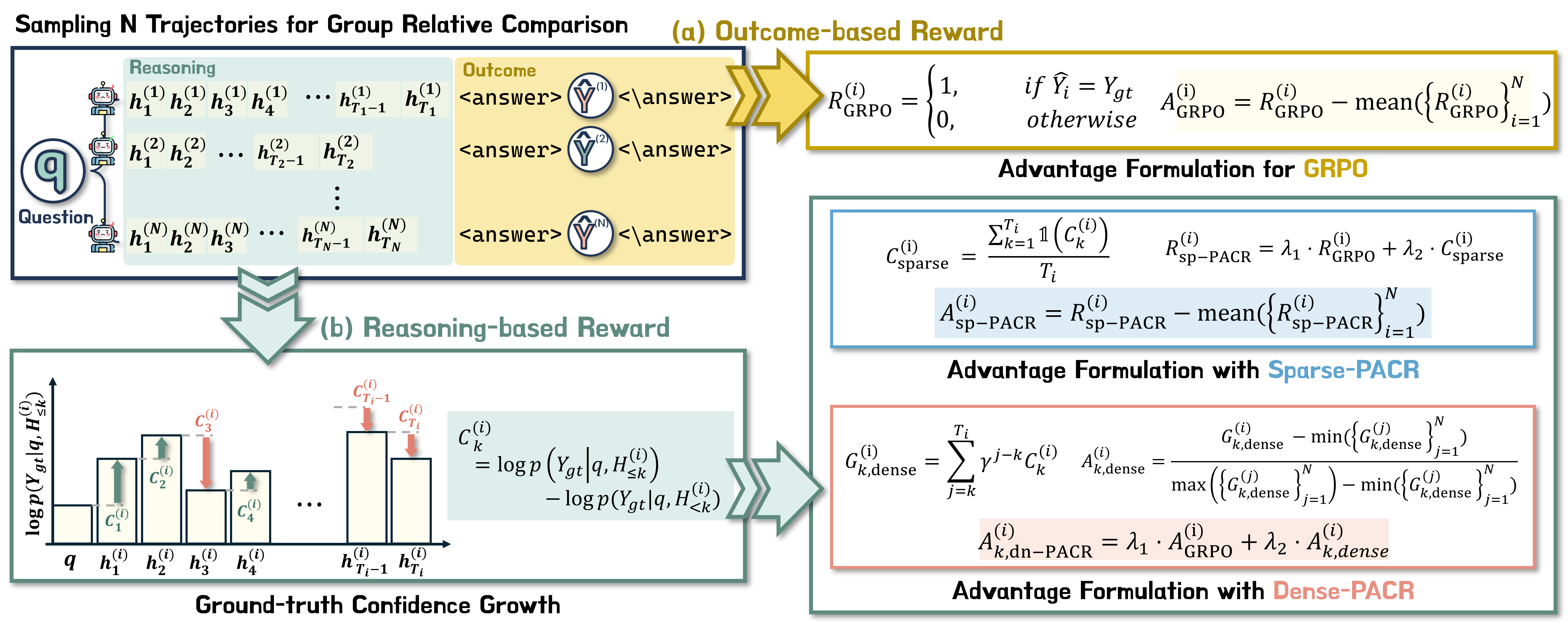}
    \vskip -0.1in
\caption{\textbf{Overview of the \textcolor[RGB]{94, 139, 126}{PACR} method and its integration with \textcolor[RGB]{160, 132, 11}{GRPO}.} Standard GRPO begins by sampling a group of $N$ reasoning trajectories for a given question. (a) A standard \textcolor[RGB]{160, 132, 11}{outcome-based reward} ($R^{(i)}_{\text{GRPO}}$) is calculated based on the correctness of the final answer. (b) Our proposed \textcolor[RGB]{94, 139, 126}{reasoning-based reward} is derived from the ground-truth confidence growth ($C_k^{(i)}$) at each step. This signal is integrated into the final advantage calculation in two ways: \textbf{\textcolor[RGB]{95, 164, 207}{Sparse-PACR}} uses the consistency of confidence growth to compute a single reward for the entire trajectory, while \textbf{\textcolor[RGB]{231, 143, 129}{Dense-PACR}} uses the magnitude of each step's gain to compute a fine-grained, per-step advantage.}

    \label{fig:method}
\end{figure}
\paragraph{\textcolor[RGB]{231, 143, 129}{Dense-PACR.}}
The Dense setting provides a more granular, step-wise reward signal. At each reasoning step $k$ in trajectory $i$, we use the ground-truth confidence gain, $C_k^{(i)}$, as an immediate reward. From this, we compute the discounted return for that step, $G^{(i)}_{k, \text{dense}}$, by summing the rewards from that point forward:
\begin{equation}
G^{(i)}_{k, \text{dense}} = \sum_{j=k}^{T_i} \gamma^{j-k} C^{(i)}_{j},
\end{equation}
where $\gamma$ is a discount factor. To create a stable, step-wise advantage signal, $A^{(i)}_{k, \text{dense}}$, we normalize these returns across the group at each step $k$. Specifically, we use Min-Max scaling to map the returns to a $[0, 1]$ range. This creates a purely positive signal that only incentivizes confidence growth without penalizing steps that do not, a design choice we validate in our ablations (Section \ref{sec:loo_ablation}). To handle trajectories of varying lengths, the discounted return $G^{(i)}_{k, \text{dense}}$ is treated as zero for any step $k$ that does not exist in trajectory $i$. The resulting advantage for a step $k$ in trajectory $i$ is then:

\begin{equation}
\label{eq: A_minmax}
A^{(i)}_{k, \text{dense}} = \frac{G^{(i)}_{k, \text{dense}} - \min_{j}(\{G^{(j)}_{k, \text{dense}}\}_{j=1}^N)}{\max_{j}(\{G^{(j)}_{k, \text{dense}}\}_{j=1}^N) - \min_{j}(\{G^{(j)}_{k, \text{dense}}\}_{j=1}^N)}.
\end{equation}
Finally, the total advantage at each step, $A^{(i)}_{k, \text{dn-PACR}}$, is the weighted sum of the trajectory-level GRPO advantage and our dense, step-wise advantage:
\begin{equation}
A^{(i)}_{k, \text{dn-PACR}} = \lambda_1 \cdot A^{(i)}_{\text{GRPO}} + \lambda_2 \cdot A^{(i)}_{k, \text{dense}},
\label{eq:dn-PACR_adv}
\end{equation}
where $A^{(i)}_{\text{GRPO}} = R^{(i)}_{\text{GRPO}} - \text{mean}(\{R^{(j)}_{\text{GRPO}}\}_{j=1}^N)$. This final advantage is then used to update the policy.

\label{sec:training}
\section{Experimental Setup}
\paragraph{Models and Baselines.}
We experiment with three open-source LLMs: Qwen2.5-Math-1.5B, Qwen2.5-Math-7B \citep{model_qwen_math}, and Qwen3-4B\footnote{For the Qwen3-4B model, we set `enable\_thinking=False' to disable its built-in chain-of-thought capabilities, allowing for a direct comparison of how our method versus standard GRPO teaches this capability.} \citep{model_qwen3}. Our baseline for all experiments is Dr.GRPO \citep{drgrpo}, a bias-mitigated version of GRPO \citep{grpo}, which we implement using the OAT framework \citep{oat}. We compare this baseline against our two proposed methods, Sparse-PACR and Dense-PACR.

\paragraph{Datasets and Evaluation.}
For training, we use the MATH dataset \citep{data_math}. Following prior work \citep{drgrpo}, we use the full dataset for the 1.5B model and filter for the more challenging levels (3-5) for the 4B and 7B models. To evaluate performance, we test our models on five diverse mathematical reasoning benchmarks: MATH500 \citep{data_math}, Minerva-Math \citep{data_minerva}, OlympiadBench \citep{data_olympiadbench}, AIME 2024, and AMC 2023 \citep{data_numinamath}. Final answers are programmatically checked for correctness using the Math-Verify \citep{mathverify} library. All results are reported as pass@1 using greedy decoding (temperature of 0).

\paragraph{Training Details.}
For each problem, we generate a group of 8 responses using sampling with a temperature of 1.0.
We report the average results across three runs with different random seeds for all experiments. All models were trained on a single node with 8 $\times$ NVIDIA H100 80GB GPUs. Further details on hyperparameters, such as learning rate and batch size, are provided in Appendix \ref{appendix:training_detail}.

\section{Results and Ablations}
\label{sec:result}

\subsection{Experimental Result}

\begin{table}[h]
    \centering
    \caption{\textbf{Results on reasoning benchmarks.} We report pass@1 accuracy across five datasets. Both Sparse-PACR and Dense-PACR consistently outperform the strong Dr.GRPO baseline across all model sizes. $\dagger$ is marked for the score reproduced and other baseline scores are from \cite{drgrpo}. The colored numbers indicate the absolute performance change relative to the Dr.GRPO baseline, with \textcolor[RGB]{50, 180, 100}{green} for improvements and \textcolor[RGB]{180, 0, 100}{red} for degradations.}
    \vskip -0.1in
    \resizebox{1.0\textwidth}{!}{
    \begin{tabular}{l|ccccc||c}
        \Xhline{4\arrayrulewidth}
       \textbf{Base model + Method}   & AIME24 & AMC & MATH500 & Minerva & OlympiadBench & \textbf{Average} \\
       \Xhline{2\arrayrulewidth}
        \rowcolor{table-yellow!60} Qwen2.5-Math-1.5B &20.0 & 32.5 & 33.0 & 12.5 & 22.8 &  24.2 \\
        R1-Distill-Qwen-1.5B @ 3k & 2.5 & 21.7 & 52.2 & 16.3 & 17.3 & 22.0 \\
        Qwen2.5-Math-1.5B-Instruct & 10.0 & 48.2 & 74.2 & 26.5 & 40.2 & 39.8 \\
       Qwen2.5-Math-1.5B + Dr.GRPO $\dagger$ & 13.3 & 47.0 & 76.8 & 32.3 & 39.0 & 41.7 \\
       \hdashline
        \rowcolor{table-green!20} Qwen2.5-Math-1.5B + Sparse-PACR & \hspace{0.63cm} 20.0 \textcolor[RGB]{0, 180, 100}{$_{\texttt{+6.7}}$} 
			& \hspace{0.63cm} 48.4  \textcolor[RGB]{0, 180, 100}{$_{\texttt{+1.4}}$} 
			& \hspace{0.63cm} 77.4  \textcolor[RGB]{0, 180, 100}{$_{\texttt{+0.6}}$}
			& \hspace{0.63cm} 29.4 \textcolor[RGB]{180, 0, 100}{$_{\texttt{-2.9}}$} 
			& \hspace{0.63cm} 37.8  \textcolor[RGB]{180, 0, 100}{$_{\texttt{-1.2}}$} 
			& \hspace{0.63cm} 42.6 \textcolor[RGB]{0, 180, 100}{$_{\texttt{+0.9}}$}  \\

        \rowcolor{table-green!40} Qwen2.5-Math-1.5B + Dense-PACR& \hspace{0.63cm} 23.3 \textcolor[RGB]{0, 180, 100}{$_{\texttt{+10.0}}$} 
			& \hspace{0.63cm} 49.4  \textcolor[RGB]{0, 180, 100}{$_{\texttt{+2.4}}$} 
			& \hspace{0.63cm} 77.4  \textcolor[RGB]{0, 180, 100}{$_{\texttt{+0.6}}$} 
			& \hspace{0.63cm} 31.7  \textcolor[RGB]{180, 0, 100}{$_{\texttt{-0.6}}$} 
			& \hspace{0.63cm} 39.0  \textcolor[RGB]{100, 100, 100}{$_{\texttt{0.0}}$} 
			& \hspace{0.63cm} 44.2 \textcolor[RGB]{0, 180, 100}{$_{\texttt{+2.5}}$}  \\

        \Xhline{3\arrayrulewidth}
        \rowcolor{table-yellow!60} Qwen2.5-Math-7B & 16.7 & 38.6 & 50.6 & 9.9 & 16.6 & 26.5 \\
        SimpleRL-Zero-7B & 26.7 & 60.2 & 78.2 & 27.6 & 40.3 & 46.6 \\
        PRIME-Zero-7B & 16.7 & 62.7 & 83.8 & 36.0 & 40.9 & 48.0 \\
        OpenReasoner-Zero-7B @ 3k & 13.3 & 47.0 & 79.2 & 31.6 & 44.0 & 43.0 \\
        R1-Distill-Qwen-7B @ 3k & 10.0 & 26.2 & 60.1 & 23.0 & 23.1 & 28.5 \\
        Qwen2.5-Math-7B-Instruct & 16.7 & 53.0 & 83.6 & 29.8 & 42.7 & 45.1 \\
        Qwen2.5-Math-7B + Dr.GRPO  $\dagger$ & 30.0 & 56.6 & 81.8 & 34.6 & 45.2 & 49.6 \\
        \hdashline
        \rowcolor{table-green!20} Qwen2.5-Math-7B + Sparse-PACR & \hspace{0.63cm} 36.7 \textcolor[RGB]{0, 180, 100}{$_{\texttt{+6.7}}$} 
			& \hspace{0.63cm} 55.4  \textcolor[RGB]{180, 0, 100}{$_{\texttt{-1.2}}$} 
			& \hspace{0.63cm} 82.6  \textcolor[RGB]{0, 180, 100}{$_{\texttt{+0.8}}$} 
			& \hspace{0.63cm} 34.6  \textcolor[RGB]{100, 100, 100}{$_{\texttt{0.0}}$} 
			& \hspace{0.63cm} 45.6  \textcolor[RGB]{0, 180, 100}{$_{\texttt{+0.4}}$} 
			& \hspace{0.63cm} 51.0 \textcolor[RGB]{0, 180, 100}{$_{\texttt{+1.4}}$}  \\

        \rowcolor{table-green!40} Qwen2.5-Math-7B + Dense-PACR & \hspace{0.63cm} 43.3 \textcolor[RGB]{0, 180, 100}{$_{\texttt{+13.3}}$} 
			& \hspace{0.63cm} 56.1  \textcolor[RGB]{180, 0, 100}{$_{\texttt{-0.5}}$} 
			& \hspace{0.63cm} 81.9  \textcolor[RGB]{0, 180, 100}{$_{\texttt{+0.1}}$} 
			& \hspace{0.63cm} 35.6  \textcolor[RGB]{0, 180, 100}{$_{\texttt{+1.0}}$} 
			& \hspace{0.63cm} 46.1  \textcolor[RGB]{0, 180, 100}{$_{\texttt{+0.9}}$} 
			& \hspace{0.63cm} 52.6 \textcolor[RGB]{0, 180, 100}{$_{\texttt{+3.0}}$}  \\

        \Xhline{3\arrayrulewidth}
        \rowcolor{table-yellow!60} Qwen3-4B & 13.3 & 32.5 & 40.2 & 9.19 & 39.4 & 26.9 \\
        Qwen3-4B + Dr.GRPO $\dagger$ & 40.0 & 63.8 & 88.4 & 33.8 & 46.8 & 54.6 \\
        \hdashline
        \rowcolor{table-green!20} Qwen3-4B + Sparse-PACR & \hspace{0.63cm}33.3\textcolor[RGB]{180, 0, 100}{$_{\texttt{-6.7}}$} & \hspace{0.63cm}67.5\textcolor[RGB]{0, 180, 100}{$_{\texttt{+3.7}}$} & \hspace{0.63cm}86.2\textcolor[RGB]{180, 0, 100}{$_{\texttt{-2.2}}$} & \hspace{0.63cm}35.3\textcolor[RGB]{0, 180, 100}{$_{\texttt{+1.5}}$} & \hspace{0.63cm}54.4\textcolor[RGB]{0, 180, 100}{$_{\texttt{+7.6}}$} & \hspace{0.63cm}55.3\textcolor[RGB]{0, 180, 100}{$_{\texttt{+0.7}}$} \\
        \rowcolor{table-green!40} Qwen3-4B + Dense-PACR & \hspace{0.63cm}46.7\textcolor[RGB]{0, 180, 100}{$_{\texttt{+6.7}}$} & \hspace{0.63cm}63.4\textcolor[RGB]{180, 0, 100}{$_{\texttt{-0.4}}$} & \hspace{0.63cm}86.8\textcolor[RGB]{180, 0, 100}{$_{\texttt{-1.6}}$} & \hspace{0.63cm}36.0\textcolor[RGB]{0, 180, 100}{$_{\texttt{+2.2}}$} & \hspace{0.63cm}55.0\textcolor[RGB]{0, 180, 100}{$_{\texttt{+8.2}}$} & \hspace{0.63cm}57.6\textcolor[RGB]{0, 180, 100}{$_{\texttt{+3.0}}$} \\
        \Xhline{4\arrayrulewidth}
    \end{tabular}
    }
    \label{tab:main_result}
\end{table}

Table \ref{tab:main_result} presents the quantitative results on various math benchmarks. For the Qwen2.5-series, we also include the instruct models at the sample scale and R1-Distill models for comparison by following \citep{drgrpo}. Our proposed reward, PACR, demonstrates significant improvements over the outcome-based reward baseline (+Dr.GRPO) in both its Sparse and Dense setting. This shows that our core method provides a positive inductive bias for improving the reasoning skills of language models.

While the sparse trajectory-level reward, Sparse-PACR, is effective on its own, we observe that Dense-PACR, which provides a more fine-grained reward, consistently achieves better performance. This highlights that enriching the training process with a dense reward signal allows the model to learn from more detailed feedback, leading to further gains in its reasoning capabilities.

\subsection{Training Curve: PACR Accelerates Exploration and Improves Convergence}
\begin{wrapfigure}{r}{0.6\textwidth}
    \vskip -0.2in
    \centering    
\includegraphics[width=0.6\textwidth]{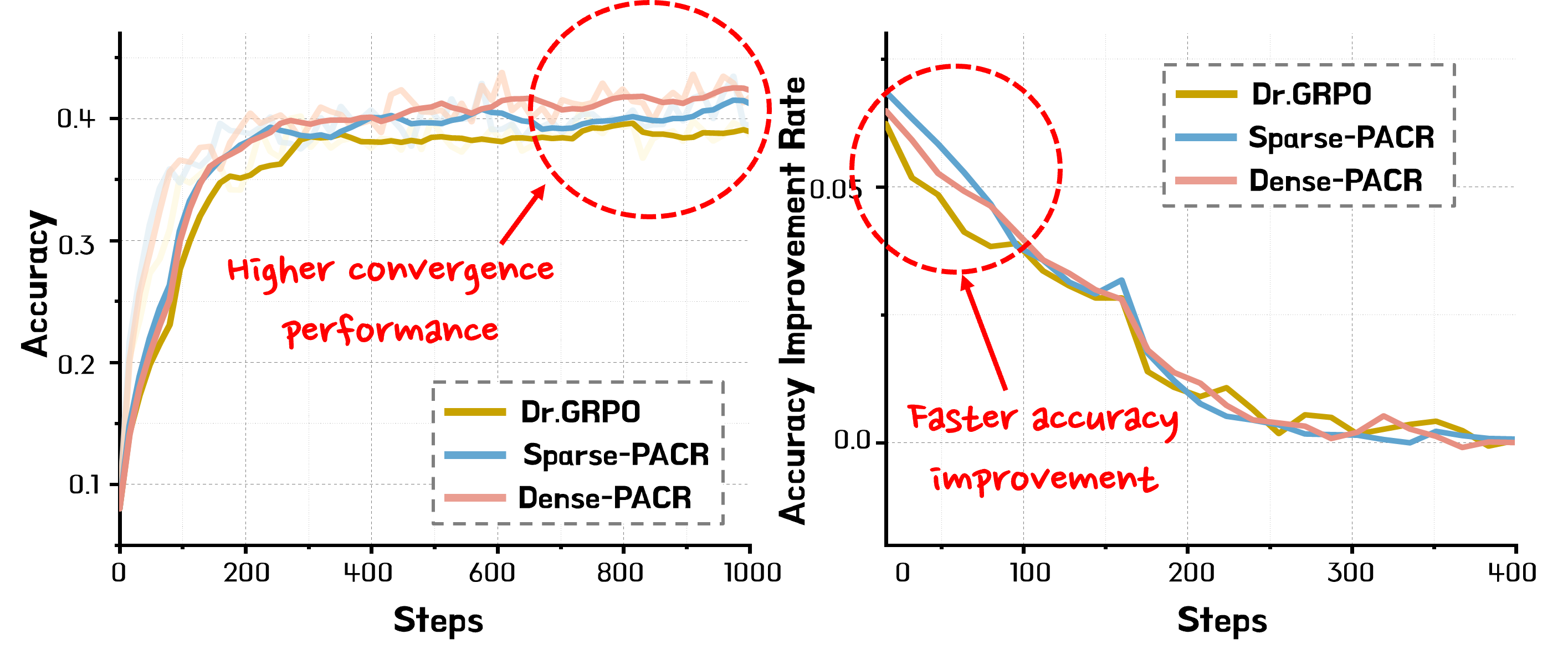}
    \vskip -0.1in
\caption{\textbf{Training dynamics for Qwen2.5-Math-1.5B.} Average pass@1 accuracy (left) and rate of accuracy improvement (right) during training. PACR methods show a faster initial rate of improvement, accelerating exploration and converging to a higher final performance.}
    \label{fig:acc_curve}
    \vskip -0.2in
\end{wrapfigure} 
Figure \ref{fig:acc_curve} illustrates the training dynamics, plotting the average pass@1 accuracy over training steps (left) and the corresponding rate of accuracy improvement (right). The right plot highlights that both PACR variants have a significantly higher rate of improvement compared to the Dr.GRPO baseline, especially during the critical early exploration phase of RL training. As shown on the left, this accelerated learning ultimately allows the PACR methods to converge to a higher final accuracy.

\subsection{Analysis on Advantage Formulation: Impact of Penalizing Intermediate Steps}
\label{sec:loo_ablation}
\begin{wrapfigure}{r}{0.40\textwidth}
    \vskip -0.3in
    \centering    
    \includegraphics[width=0.35\textwidth]{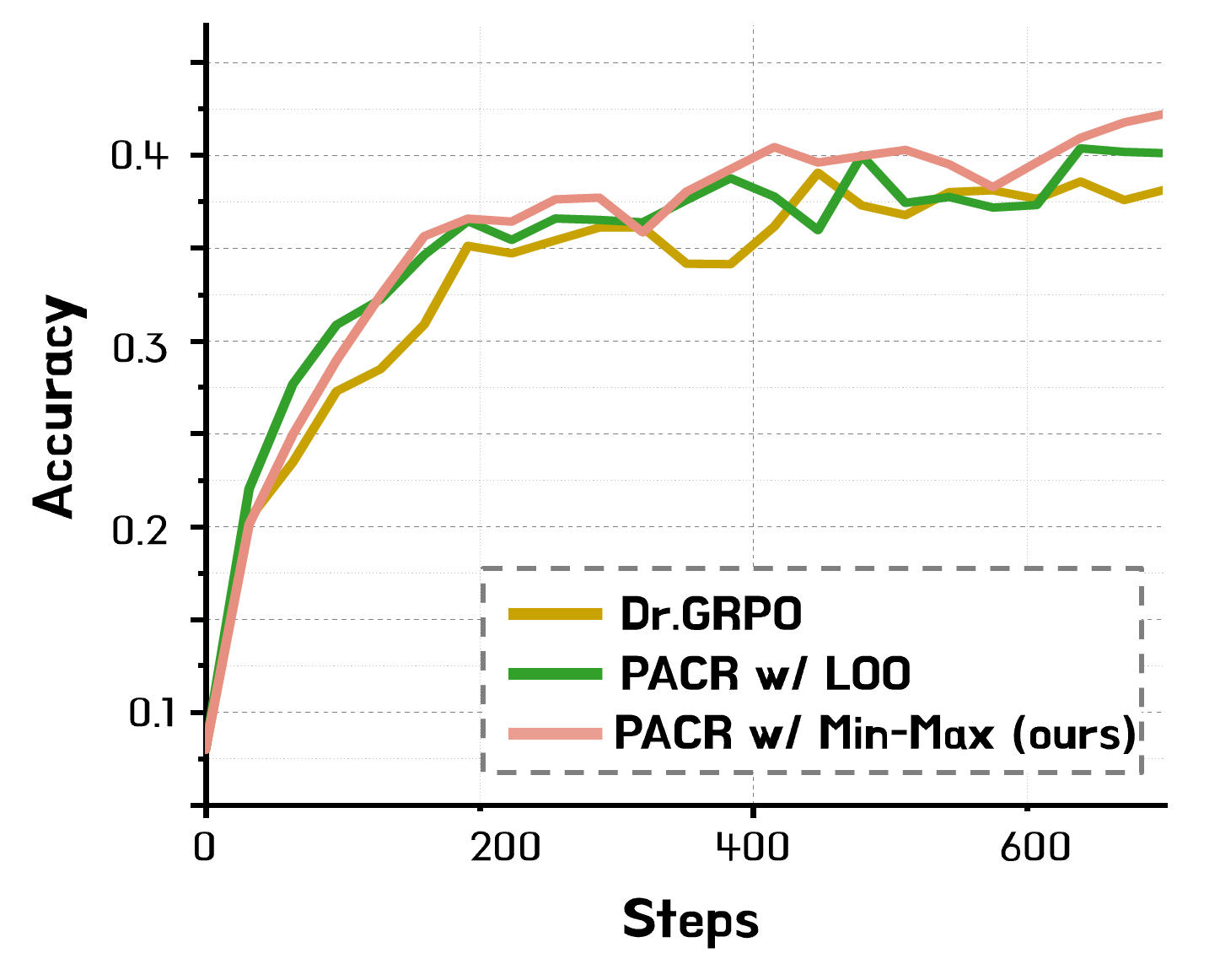}
    \vskip -0.1in
\caption{\textbf{Advantage Normalization.} Comparison of Min-Max normalization and a Leave-One-Out (LOO) baseline for Dense-PACR.}
    \label{fig:acc_curve_analysis}
    \vskip -0.3in
\end{wrapfigure} 

In this section, we analyze the impact of the advantage formulation in our Dense-PACR setting. A crucial design choice is how to normalize the raw discounted returns ($G^{(i)}_{k, \text{dense}}$) into a stable advantage signal. We compare our Min-Max normalization against a widely used Leave-One-Out (LOO) baseline \citep{rloo, prime}.

The key difference is that the LOO baseline centers the returns, which can assign \textbf{negative advantages} that penalize steps with below-average confidence gains:
\begin{equation}
A_{k, \text{loo}}^{(i)} = G^{(i)}_{k, \text{dense}} - \text{mean}(\{G_{k, \text{dense}}^{(j)}\}_{j=1, j\neq i}^N).
\end{equation}

In contrast, our Min-Max normalization (Eq. \ref{eq: A_minmax}) scales returns to a $[0, 1]$ range, creating a \textbf{purely positive signal} for the reasoning process that only rewards confidence growth.

Figure \ref{fig:acc_curve_analysis} shows this design choice has a clear impact on the training dynamics. The penalizing nature of the LOO baseline initially accelerates learning by aggressively pruning suboptimal steps, but this leads to premature convergence and a performance plateau. Conversely, our non-penalizing Min-Max approach encourages more sustained exploration, ultimately converging to a higher final accuracy. With our method, process-level penalization is avoided; a negative training signal is only applied by the main GRPO reward when the model produces a definitively incorrect final answer.

\section{Conclusion}
In this work, we addressed the limitations of sparse, outcome-based rewards in RLVR by introducing the Progressively Ascending Confidence Reward (PACR), a dense, model-intrinsic signal derived from the model's evolving belief in the ground-truth answer. Through a series of empirical observations and a formal theoretical proof, we validated that confidence growth serves as a powerful inductive bias, effectively constraining the search space to regions richer in logically sound and faithful reasoning paths. Our experiments demonstrated that augmenting GRPO with PACR not only accelerates training but also converges to a higher final performance across multiple reasoning benchmarks, with the fine-grained Dense-PACR variant proving most effective. Ultimately, our work shows that informative, dense rewards for complex reasoning can be effectively extracted from the internal dynamics of the learning policy itself, suggesting a promising direction for creating more effective and reliable methods for fine-tuning the reasoning capabilities of large language models.
\bibliography{main}
\bibliographystyle{iclr2026_conference}
\newpage
\appendix
\section{Appendix}
\subsection{Limitations and Future Work} 
While our study demonstrates that Progressively Ascending Confidence Reward (PACR) provides a powerful inductive bias for mathematical reasoning, a limitation is that our work is primarily confined to language models. Therefore, a promising direction for future work is to investigate the efficacy of the proposed PACR framework in multimodal reasoning tasks, such as visual math problems, using Vision Language Models (VLMs).

\subsection{Broader Impact}
This work introduces a new inductive bias designed to improve the reasoning capabilities of Large Language Models. By leveraging the model's intrinsic confidence dynamics, our method provides fine-grained, step-level supervision without the significant overhead of training separate reward models or requiring manual data annotation. By eliminating the need for external process-reward models or human-annotated datasets, this research significantly lowers the computational and financial barriers to entry for training sophisticated reasoning agents. 

\subsection{The Use of LLMs}\label{sec:llm}
We used LLMs solely for light editing such as correcting grammatical errors and polishing some words. They did not contribute to research ideation, experiments, analysis, or substantive writing. We have reviewed all AI-assisted edits and take full responsibility for the final content of this paper.

\subsection{Ethic Statement}
This research adheres to the highest standards of academic integrity. All existing work is appropriately cited, and this paper does not violate the use of others’ work without reference. The experiments conducted do not introduce new datasets or utilize any sensitive data related to demographic or identity characteristics. 

\subsection{Training Details}
\label{appendix:training_detail}
We present the details of our training configuration as follows. We use a total batch size of 128 and perform one PPO epoch per rollout. The per-device batch size is set to 4 for Qwen2.5-Math-1.5B, and 2 for both Qwen2.5-Math-7B and Qwen3-4B. During rollouts, we use a sampling temperature of 1.0 and generate 8 rollouts per prompt. For optimization, we use the AdamW optimizer \citep{adamw} with a constant learning rate of 1e-6, without warmup or scheduler. The maximum prompt and generation lengths are set to 1024 and 3000 tokens, respectively. For the KL penalty, we set the coefficient $\beta =0$, effectively deactivating it during training. For the $\lambda_1$, and $\lambda_2$, we search in the range of [1, 0.99, 0.9, 0.8, 0.5] and [0.01, 0.1, 0.2, 0.5], and for the both sparse and dense setting, $\lambda_1$ and $\lambda_2$ are set to 0.9, and 0.1, respectively

\subsection{Prompt used for Observation}
\label{appendix:prompt}
To analyze the coherence of the reasoning paths (Observation 2) and the correlation between the large stepwise confidence gain and the pivotal reasoning step (Observation 3) in Section \ref{sec:observation}, we utilize GPT-5 as an evaluator. The prompts used to evaluate the reasoning steps for these respective observations are shown in Figures \ref{fig:reasoning_eval_prompt} and \ref{fig:pairwise_eval_prompt}.

\begin{figure*}[h]
\begin{center}
\centerline{\includegraphics[width=0.9\linewidth]{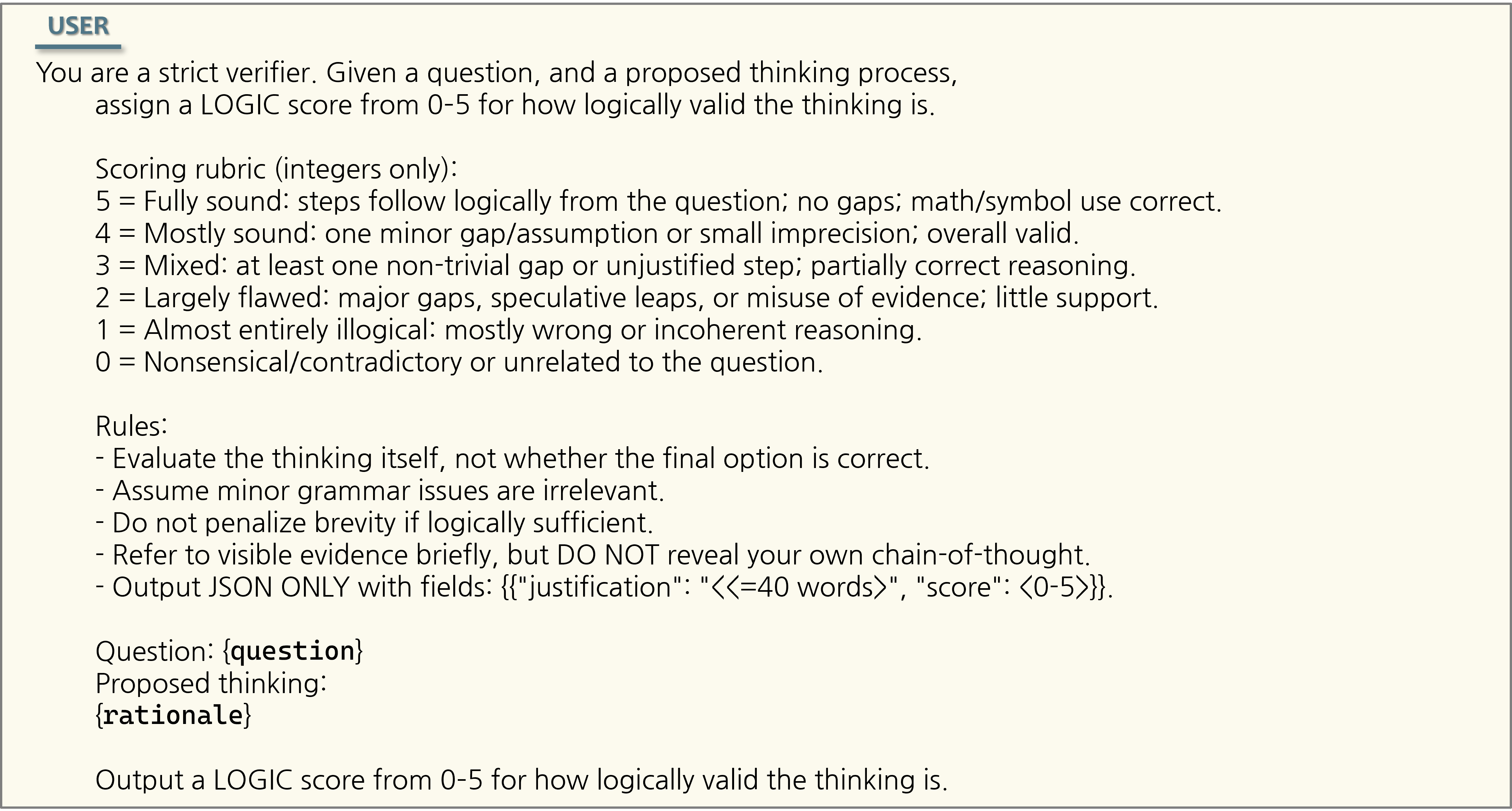}}
\caption{\textbf{Prompt used to evaluate reasoning quality for Observation 2.}}
\label{fig:reasoning_eval_prompt}
\end{center}
\end{figure*}

\begin{figure*}[h]
\begin{center}
\centerline{\includegraphics[width=0.9\linewidth]{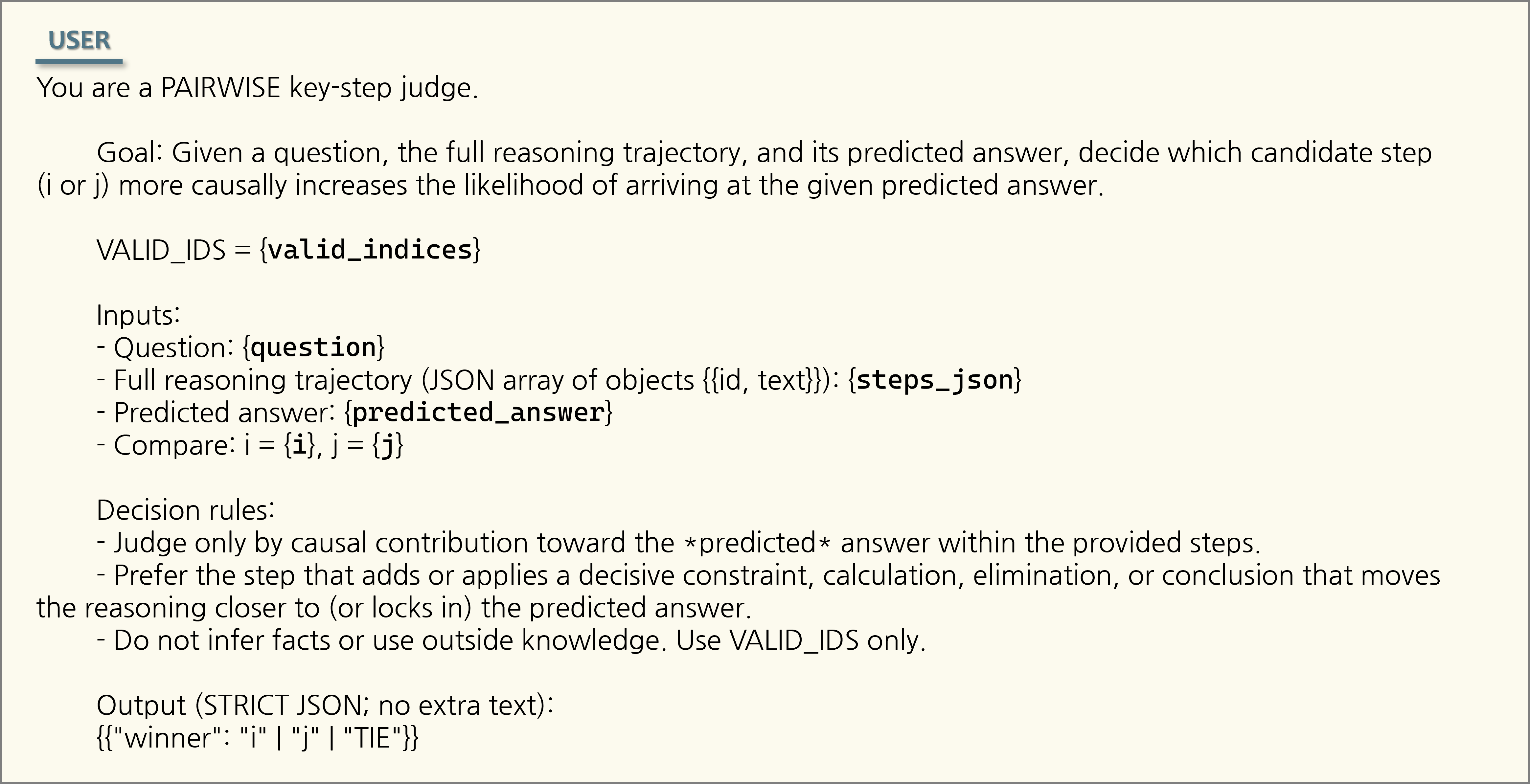}}
\caption{\textbf{Prompt used to evaulate pairwise comparison for the impactful reasoning steps in Observation 3.}}
\label{fig:pairwise_eval_prompt}
\end{center}
\end{figure*}

\subsection{Examples for Observation 3}
\label{appendix:obs3_qual}
This section provides additional qualitative examples that support the central claim of Observation 3. As illustrated by the reasoning trajectories from Qwen2.5-Math-7B (Figure \ref{fig:obs3_appendix_example_qwen}) and GPT-OSS-20B (Figure \ref{fig:obs3_appendix_example_gptoss1}, \ref{fig:obs3_appendix_example_gptoss2} and \ref{fig:obs3_appendix_example_gptoss3}), large positive spikes in the stepwise confidence gain $C_k$ consistently align with pivotal problem-solving steps, such as applying a key formula or executing a critical calculation.

\begin{figure*}[h]
\begin{center}
\centerline{\includegraphics[width=0.9\linewidth]{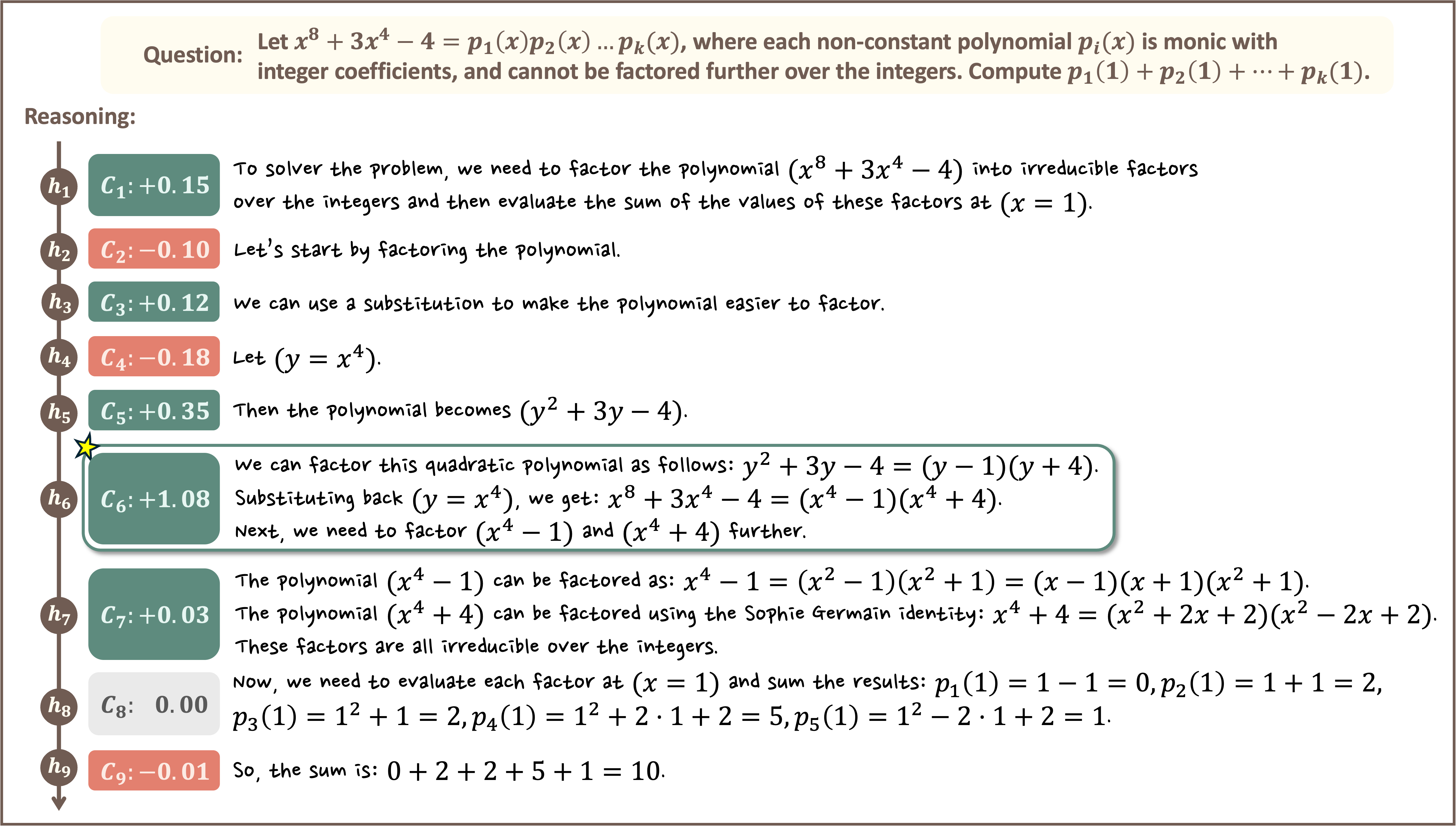}}
\caption{\textbf{Example trajectory from Qwen2.5-Math-7B for a polynomial factorization problem.}\\The model tackles the 8th-degree polynomial by reducing it to a quadratic form, then applying systematic factorization to break it down into irreducible integer-coefficient components. The pivotal moment occurs at step $h_6$, where the model executes the core factorization. By successfully factoring the simplified quadratic expression $y^2+3y-4$ and subsequently substituting $x^4$ back in for $y$, the model achieves the central algebraic decomposition of the original polynomial into two more tractable factors ($x^4-1$ and $x^4+4$). This step represents the main breakthrough required to solve the problem, and the corresponding confidence gain ($C_6=+1.08$) indicates the model's recognition that the most significant hurdle in the factorization process has been overcome.}
\label{fig:obs3_appendix_example_qwen}
\end{center}
\end{figure*}

\begin{figure*}[h]
\begin{center}
\centerline{\includegraphics[width=0.9\linewidth]{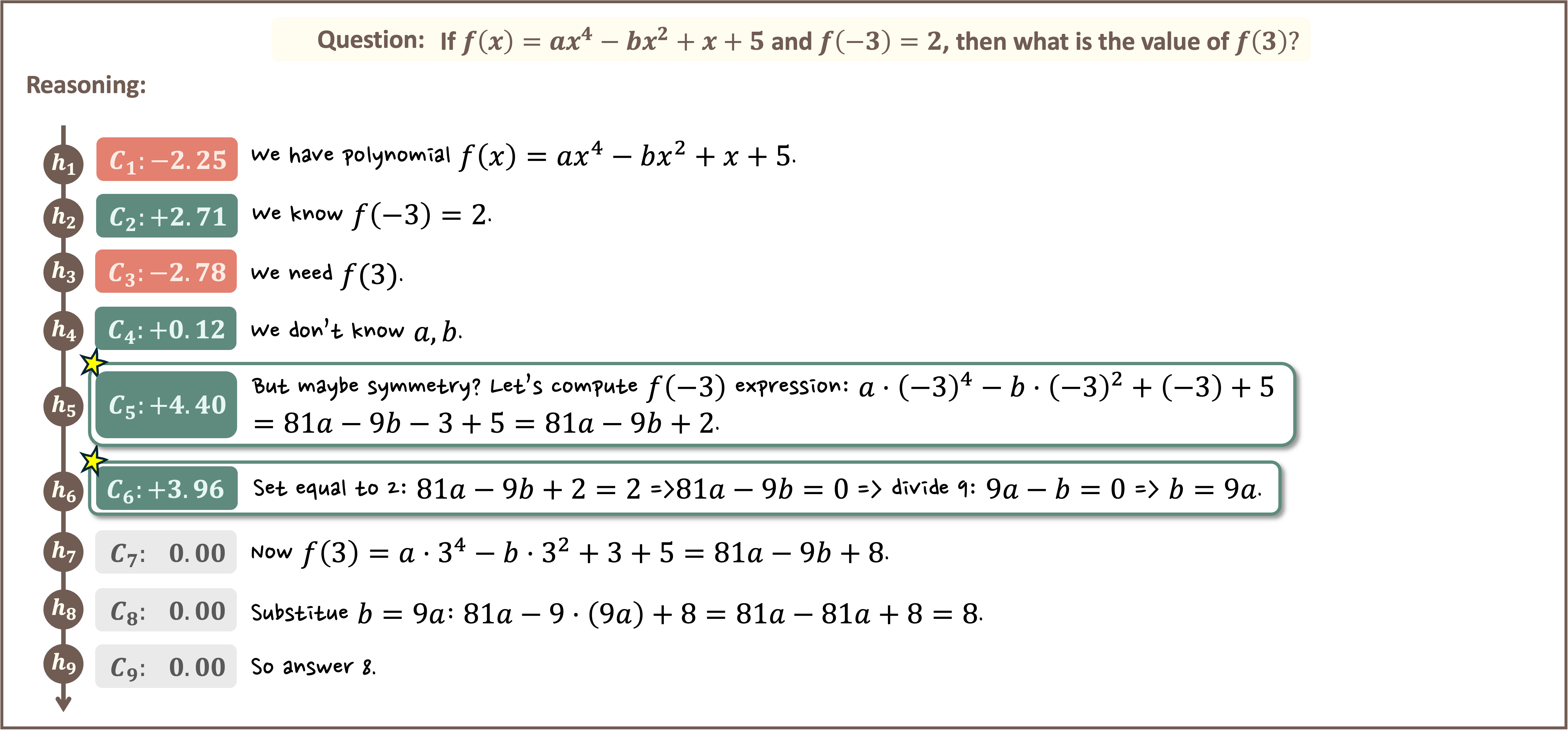}}
\caption{\textbf{Example trajectory from GPT-OSS-20B for a polynomial symmetry problem.}\\The problem appears ostensibly unsolvable due to the unknown coefficients $a$ and $b$. The critical insight emerges across steps $h_5$ and $h_6$, where the model leverages the inherent symmetry of the function's even-powered terms ($ax^4, -bx^2$) and utilizes the given condition $f(-3) = 2$. By evaluating the expression for $f(-3)$ and equating it to $2$, the model uncovers the essential relationship between the unknown coefficients ($b = 9a$). This discovery is the pivotal moment that unlocks the entire problem, as it enables the cancellation of the unknown terms in the subsequent calculation of $f(3)$. 
}
\label{fig:obs3_appendix_example_gptoss1}
\end{center}
\end{figure*}

\begin{figure*}[h]
\begin{center}
\centerline{\includegraphics[width=0.9\linewidth]{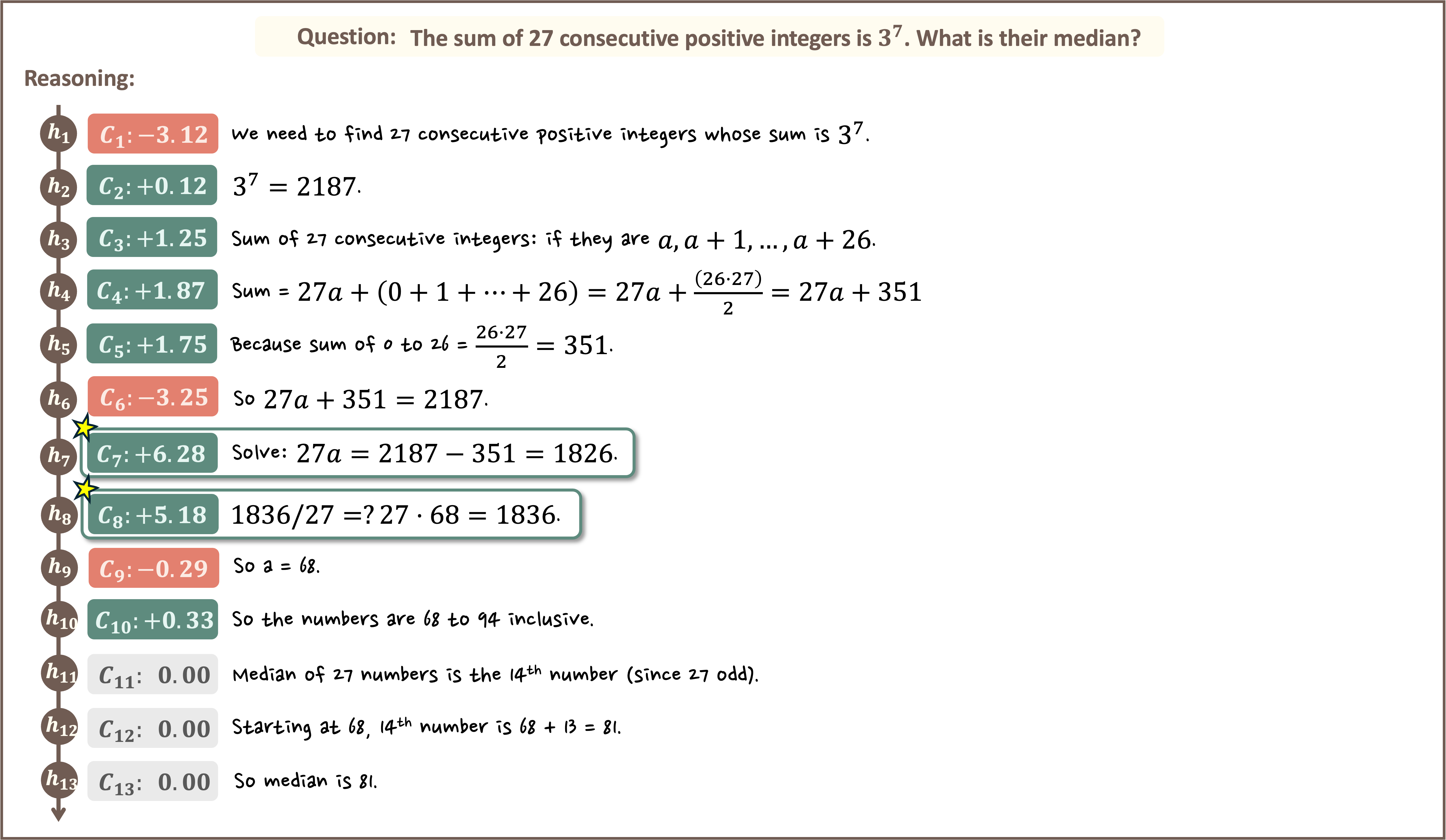}}
\vskip -0.1in
\caption{\textbf{Example trajectory from GPT-OSS-20B for a consecutive integer median problem.}\\The initial phase of this problem involves correctly establishing an algebraic equation for the sum of the integers: $27a + 351 = 2187$. The pivotal moment is a two-step process spanning $h_7$ and $h_8$. In $h_7$, the model transitions from algebraic formulation to the decisive computation to solve for the first integer, $a$, by successfully isolating the term $27a$. This is immediately followed by the critical execution in $h_8$, where the actual division is performed to find the explicit value of $a$. Securing the value of $a$ is the primary breakthrough, as it provides the key to determining all numbers in the sequence and thus the median. 
}
\label{fig:obs3_appendix_example_gptoss2}
\end{center}
\end{figure*}

\begin{figure*}[h]
\begin{center}
\centerline{\includegraphics[width=0.9\linewidth]{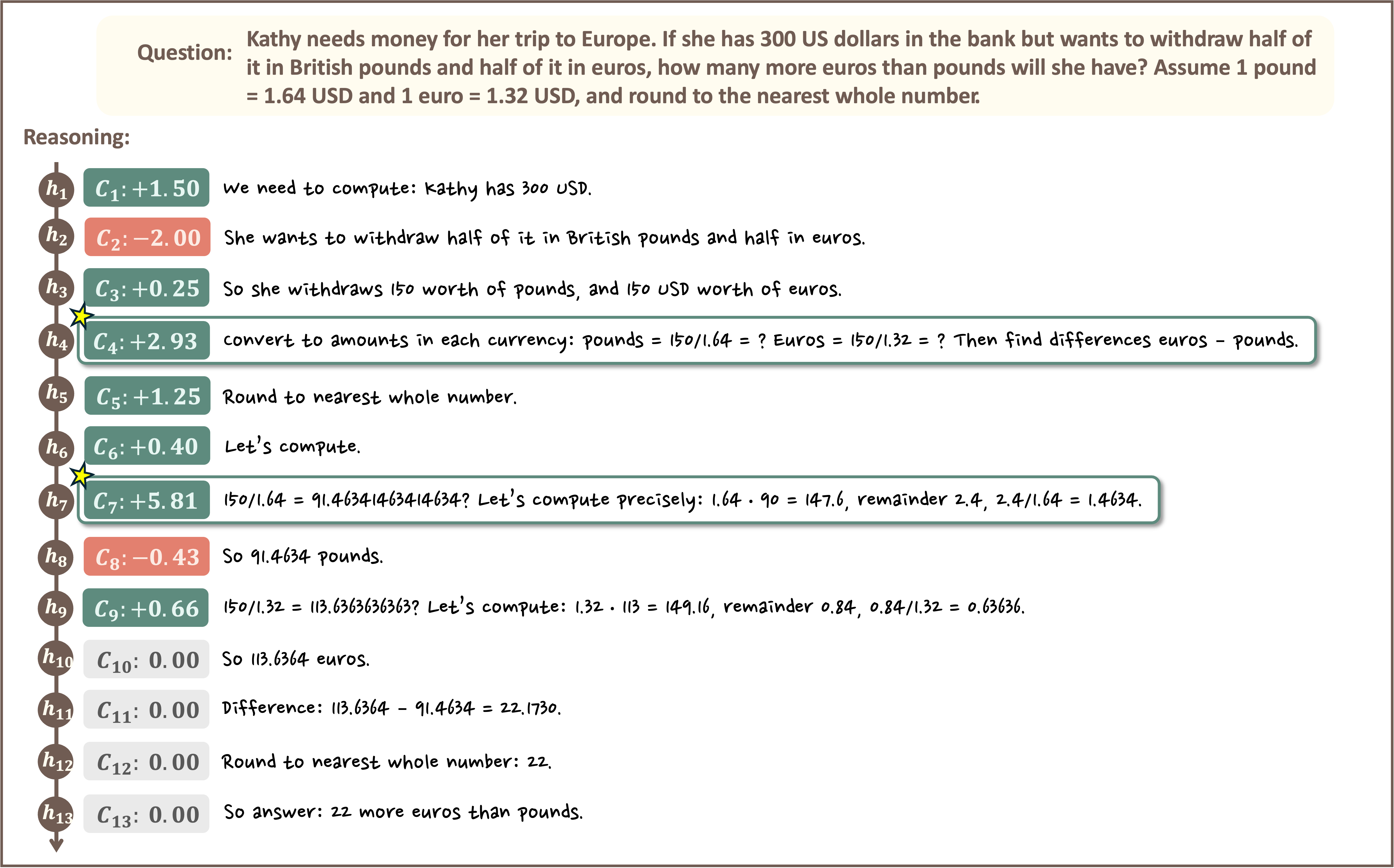}}
\vskip -0.1in
\caption{\textbf{Example trajectory from GPT-OSS-20B for a currency exchange problem.}\\ This reasoning trajectory features two pivotal moments. First, step $h_4$ serves as a critical planning phase, where the model correctly formulates the computational roadmap required for the solution: two currency conversions via division, followed by a subtraction. This demonstrates a comprehensive understanding of the problem's logic. The second, more significant pivotal moment occurs at the execution phase in step $h_7$, where the model accurately performs the first of the two required divisions. Successfully clearing this key computational hurdle provides the model with high confidence ($C_7 = +5.81$) that its strategy is effective and the path to the final answer is now clear.}
\label{fig:obs3_appendix_example_gptoss3}
\end{center}
\end{figure*}

\end{document}